\documentclass[journal]{IEEEtran}

\usepackage{graphicx}
  \graphicspath{{Figures/}}
  \DeclareGraphicsExtensions{.pdf,.jpg,.png,.eps}
\usepackage{amsmath,amssymb}
\usepackage{xspace}
\usepackage{booktabs}
\usepackage{amsthm}
\usepackage{wrapfig}
\usepackage{dsfont}
\usepackage{hyperref}
\usepackage{arydshln}

\newcommand{\secref}[1]{Section~\ref{#1}}
\newcommand{\secrefb}[1]{(Section~\ref{#1})}
\renewcommand{\eqref}[1]{(\ref{#1})}
\newcommand{\eqrefb}[1]{(Eqn. \ref{#1})}
\newcommand{\figref}[1]{Fig.~\ref{#1}}
\newcommand{\figrefb}[2]{(Fig.~\ref{#1}#2)}

\newcommand{\ie}{\textit{i.e.}}
\newcommand{\eg}{\textit{e.g.}}
\newcommand{\etc}{\textit{etc.}}

\newcommand{\rvs}[1]{{\color{black}#1}}

\newtheorem{theorem}{Theorem}

\DeclareMathOperator*{\argmax}{arg\,max} 
\newcommand{\expect}{\mathbb{E}} 

\newlength{\citeskipup}
\newlength{\citeskipdown}

\usepackage{color}
\definecolor{fullred}{rgb}{0.95,.0,.1} 
\newcounter{cmt}

\newcommand{\Fbox}[1]{\setlength{\fboxrule}{1pt}\setlength{\fboxsep}{0pt}\fbox{#1}}

\newcommand{\scenario}{\ensuremath{{\boldsymbol{\phi}}}\xspace}

\newcommand{\scenariosetnode}[1]{\ensuremath{{\boldsymbol{\Phi}_{#1}}}\xspace}

\newcommand{\randnum}{\ensuremath{\varphi}\xspace}

\newcommand{\obs}{\ensuremath{z}\xspace}
\newcommand{\act}{\ensuremath{a}\xspace}

\newcommand{\optact}{\ensuremath{a^*}\xspace}

\newcommand{\rfun}[2]{\ensuremath{R({#1},{#2})}\xspace}

\newcommand{\obsset}{\ensuremath{Z}\xspace}

\newcommand{\sinit}{\ensuremath{s_0}\xspace}

\newcommand{\node}{\ensuremath{b}\xspace}

\newcommand{\newnode}{\ensuremath{b'}\xspace}

\newcommand{\rootnode}{\ensuremath{b_0}\xspace}

\newcommand{\polvalue}{\ensuremath{V}\xspace}

\newcommand{\nvisit}[1]{\ensuremath{N({#1})}\xspace}
\newcommand{\nvisita}[2]{\ensuremath{N({#1},{#2})}\xspace}
\newcommand{\ubsymbol}{\ensuremath{u}\xspace}
\newcommand{\ub}[1]{\ensuremath{\ubsymbol({#1})}\xspace}
\newcommand{\uba}[2]{\ensuremath{\ubsymbol({#1},{#2})}\xspace}

\newcommand{\lbsymbol}{\ensuremath{l}\xspace}
\newcommand{\lb}[1]{\ensuremath{\lbsymbol({#1})}\xspace}

\newcommand{\gapsymbol}{\ensuremath{\epsilon}\xspace}
\newcommand{\regretsymbol}{\ensuremath{\hat{\epsilon}}\xspace}

\newcommand{\optdespotpolicy}{\ensuremath{\hat{\pi}^*}\xspace}
\newcommand{\optpolicy}{\ensuremath{\pi^*}\xspace}

\newcommand{\algname}{LeTS-Drive\xspace}

\newcommand{\task}{crowd-driving\xspace}

\newcommand{\olimitation}{Open-SSL\xspace}
\newcommand{\climitation}{Closed-SSL\xspace}
\newcommand{\clreinforce}{Closed-RL\xspace}
\newcommand{\sac}{PG\xspace}

\usepackage{fancyhdr}

\fancypagestyle{firstpage}{%
  \lhead{Under revision with IEEE Transactions on Robotics}
}

\begin{document}

\title{\huge \textbf{Closing the Planning-Learning Loop with \\ Application to Autonomous Driving}}

\author{
\IEEEauthorblockN{
Panpan Cai and
David Hsu,~\IEEEmembership{Fellow,~IEEE}}

\IEEEauthorblockA{School of Computing, National University of Singapore, 117417 Singapore}
}

\IEEEtitleabstractindextext{%

\begin{abstract}

Real-time planning under uncertainty is critical for robots operating in complex dynamic environments. 
Consider, for example, an autonomous robot vehicle driving in dense,  unregulated urban traffic of cars, motorcycles, buses, \etc.
The robot vehicle has to plan in both short and long terms, in order  
to interact with many  traffic participants of uncertain intentions and drive effectively. 
Planning explicitly over a long time horizon, however, incurs prohibitive computational cost and is impractical under real-time constraints. To achieve real-time performance for large-scale planning,  this work  introduces a new algorithm \textit{Learning from Tree Search for Driving} (\algname), which integrates planning and learning in a closed loop, and applies it to autonomous driving in crowded urban traffic in simulation. 
Specifically, \algname learns a policy and its value function from \rvs{data provided by} an online planner, which searches a sparsely-sampled belief tree; the online planner in turn uses the learned policy and value functions  as heuristics to scale up its run-time performance for real-time robot control.
These two steps are repeated to form a closed loop so that the planner and the learner inform each other and improve in synchrony.
  The algorithm learns on its own in a self-supervised manner, without human effort on explicit data labeling. 
  Experimental results  demonstrate that \algname outperforms either planning or learning alone, as well as open-loop integration of planning and learning. 

\end{abstract}

}
\maketitle
\thispagestyle{firstpage}

\IEEEdisplaynontitleabstractindextext
\IEEEpeerreviewmaketitle

\section{Introduction} \label{sec:introduction}

As robots move closer to our daily lives in offices, homes, or on the road, a major challenge is tackling complex, highly dynamic, and interactive environments in real time.
One example is \textit{\task}: an autonomous vehicle drives through crowded roads and uncontrolled intersections, with heterogeneous traffic flows of cars, motorcycles, buses, \etc~(\figref{fig:meskel}). The many traffic participants  act aggressively to compete for the passageway and avoid collisions, leading to complex interactions and sometimes chaotic traffic.  To drive effectively in such an environment, the robot vehicle must perform \textit{long-term} planning in order to hedge against potential hazards in the future and balance short-term and long-term risks.  The primary challenge is the \textit{scalability} of planning in high-dimensional state spaces: for crowd-driving, the world state is the cross-product of the individual states of the ego-vehicle and many traffic participants nearby. The challenge compounds with \textit{uncertainties}, as a result of complex environment dynamics, as well as imperfect robot control and sensing.

\begin{figure}[!t]
  \begin{center}
    \includegraphics[height=0.19\textwidth]{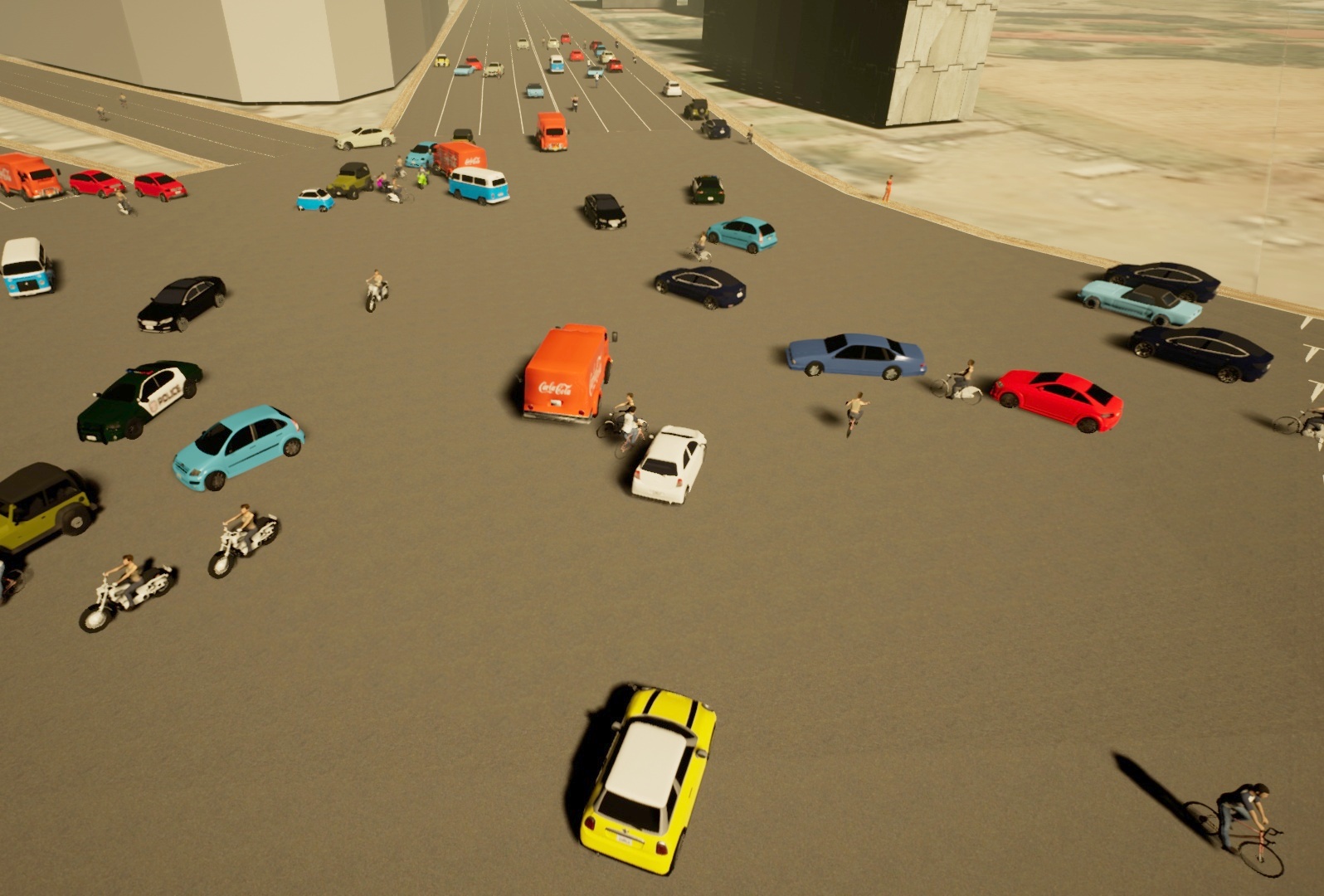}
    \includegraphics[height=0.19\textwidth]{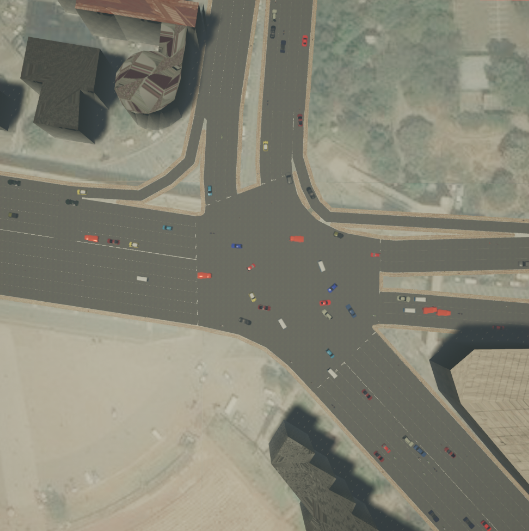}
  \end{center}
  	\vspace{-0.2cm}
	\caption{Crowd-driving. Drive through dense, unregulated,  heterogeneous traffic of cars, motorcycles, buses, pedestrians, \ldots\, in complex urban environments.}
	\vspace{-0.1cm}
    \label{fig:meskel}
\end{figure}

One common approach to real-time planning  is to perform online look-ahead search under a suitable model.
The challenge of long-term planning then depends directly on the search horizon,~$H$. With increasing $H$, the size of the search tree  grows exponentially, quickly breaching the limit of real-time computation. Further, the model error, if any, accumulates and eventually results in sub-optimal decisions. Naturally, we ask: \textit{can we reap the benefits of long-term planning without a deep search?}  

To tackle this challenge, we propose \textit{Learning from Tree Search for Driving} (\algname), which integrates planning and learning in a closed loop. 
The algorithm comprises two key ideas:
\begin{itemize}
    \item plan locally and learn globally, and
    \item close the planning-learning loop.
\end{itemize}
The algorithm consists of two interacting components, an online planner and a learner.
 The planner plans \textit{locally}, through online look-ahead search with a short horizon, and
 relies on learned heuristics---policy and value neural networks trained from data---as \textit{global} approximations to guide the search.
In parallel, the  learner gathers experiences from the planner and uses the data to  improve the  policy and value networks continuously. It feeds the improved heuristics back to the planner, thus \textrm{closing the planning-learning loop} and improving both planning and learning in synchrony. As a result, \rvs{\algname learns the heuristics both \textit{from} and \textit{for} online planning  in a self-supervised manner, without human effort on explicit data labeling.}
See \figref{fig:overview} for an illustration.

\rvs{\algname uses the  \emph{partially observable Markov decision process}
(POMDPs) as a model of uncertainties.}
We build our planner on top of HyP-DESPOT \cite{Hyp-despot}, a
state-of-the-art online POMDP planning algorithm, by integrating it with
learned policy and value networks.
\rvs{\algname greatly
improves the scalability of online planning under uncertainty.} Furthermore, it provides
theoretical guarantee on the near-optimality of its decisions, despite using heuristics learned approximately. 

\algname is flexible and can take advantage of both self-supervised and
reinforcement learning.
The self-supervised learner fits the policy network
and the value network directly to the planner outputs. It then iteratively
updates the policy, using tree search as the policy improvement
operator.
The reinforcement learner learns a policy network from environmental
feedback directly, treating the planner as an off-policy actor to provide high-quality
exploration and reward.

The underlying idea of \algname aligns in spirit with AlphaGO-Zero
\cite{AlphaGOZero}, which uses learning-guided game tree search.  AlphaGO-Zero has
beat the human world champion of GO, a perfect-information two-player board
game. However, real-life robotics tasks, such as driving in dense traffic,
pose the new challenges of partial observability, complex dynamics, and
interaction with many heterogeneous agents.  The resulting uncertainties are
major obstacles to scalability.
 
We  evaluate \algname in a realistic simulator, \mbox{SUMMIT} \cite{SUMMIT},
which simulates dense, unregulated urban traffic at worldwide locations.
Given any urban map supported by the OpenStreetMap \cite{OSM}, SUMMIT automatically
generates realistic traffic, using GAMMA \cite{luo2022gamma}, a recently developed traffic
model that has been validated on multiple real-world datasets.
Our results show that by integrating planning and learning, LeTS-Drive
significantly outperforms either planning or learning alone.
Further, closed-loop integration
enables significantly faster learning and better asymptotic performance than
open-loop integration.  After training, \algname exhibits sophisticated driving behaviors
in dense, chaotic traffic, and generalizes to diverse environments.



\section{Background}\label{sec:background}

\subsection{Online POMDP Planning}
Planning under uncertainty is critical for robust robot performance in complex, dynamic environments. A key challenge is \textit{partial observability}: true system states are not known accurately and only  revealed partially from sensor observations. A principled solution is belief-space planning: maintain a \emph{belief}, a probability distribution, over possible system states; then, conditioned on the belief, predict possible future states and observations, and optimize the robot's control policy in simulated hindsight. This process is formalized as the partially observable Markov decision process (POMDP) \cite{Kaelbling_1998}.

Formally, a POMDP model is represented as a tuple $(S,A,Z,T,O,R)$, where $S$ represents the state space of the world, $A$ denotes the space of possible actions, and $Z$ represents the space of observations. $T$, $O$, and $R$ denote the transition function, observation function, and the reward function, respectively. Concretely, the model assumes a discrete-time Markovian random process.
When the robot takes an action $a$ at a state $s$, it assumes the world transits to a new state $s'$ at the next time step with a probability $p(s'|s,a)=T(s,a,s')$. After that, the robot receives an observation $z$ with probability $p(z|s',a)=O(s',a,z)$, together with a real-valued reward $R(s,a)$.

To plan, the robot maintains a belief \node, a probability distribution over $S$, where $b(s)$ denotes the probability of the robot \rvs{being} in state $s$. POMDP planning searches for a belief-space \emph{policy} $\pi\colon \mathcal{B}\rightarrow A$, which prescribes for each belief $b$ in the belief space $\mathcal B$ an action $a$ that optimizes future values. For infinite horizon POMDPs, the \emph{value} of a policy $\pi$ at a belief $b$ is defined as the \textit{expected total discounted reward} over time, achieved by executing the policy $\pi$ from $b$:
\begin{equation}
 V_{\pi}(b)=\expect\left(\sum\limits_{t=0}^{\infty} \gamma^t R(s_t, \pi(b_t))  \,\bigg{|}\,b_0=b\right),
 \label{eqn:value}
\end{equation}
where  $\gamma\in[0,1)$ is a discount factor and \rvs{$t$ is the time step}.

Complex tasks are often solved using online planning: at each time step, the planner computes an optimal action $a^*$ for the current belief $b$, executes it immediately, and re-plans in the next time step. 
Online planning is usually performed using \emph{belief tree search}. The search starts from the current belief and constructs a tree consisting of all reachable beliefs in the future. This is achieved using \textit{Monte Carlo (MC) simulations}---the robot takes an action at the current state, receives an observation of the outcome, then takes the next action, so on and so forth.
A \textit{full belief tree} considers all possible outcomes of MC simulations. It \textit{recursively} branches with all feasible actions then all possible observations, until reaching a maximum planning horizon $H$. 


The desired output of belief tree search is an optimal policy, $\pi^*$, that maximizes the value at the current belief \rootnode: 
\begin{eqnarray}\label{eqn:optpol}
\pi^*&=&\argmax_{\pi\rvs{\in \Pi}} V_{\pi}(\rootnode),
\end{eqnarray}
\rvs{where $\Pi$ denotes the set of all possible closed-loop policies.}
The optimal policy is computed by applying the Bellman's operator to all nodes in the belief tree:
\begin{eqnarray}\label{eqn:backup}
\fontsize{9}{10}
\polvalue(\node)&=&\max_{\act \in A} \polvalue(\node, \act)\label{eqn:bellman1}\\
\polvalue(\node, \act)&=& \rfun{\node}{\act}+\gamma
	\sum_{\obs \in \obsset}p(\obs|\node,\act)\polvalue(\newnode) \label{eqn:bellman2}
	\end{eqnarray}
\noindent 
where \node is a belief node and \newnode is a child of \node produced by an action-observation pair, $a$ and $z$, through \rvs{Bayesian belief update \eqrefb{eqn:bayes_rule}}.
The first term of Eqn. \eqref{eqn:bellman2}, $\rfun{\node}{\act}=\sum_{s\in S}R(s,a)\rvs{b(s)}$, calculates the expected immediate reward of taking action \act at belief $b$. The second considers the expected long-term value marginalized over child beliefs, 
where,
\begin{eqnarray}\label{eqn:p_bza}
	p(\obs|\node,\act)&=&\sum_{s'\in S} O(s',a,z)\sum_{s\in S} T(s,a,s')b(s),
\end{eqnarray}
\vspace{-0.2cm}

\noindent represents the probability of observing the relevant \obs after taking action \act at \node.

A naive approach of optimal planning is to perform dynamic programming in the belief tree, using Eqn. \eqref{eqn:bellman1} and \eqref{eqn:bellman2} as the backup operator. 
After planning, the robot executes the optimal action at the root.
Then, it updates the current belief according to the action $a_t$ taken and the observation $z_t$ received, using a Bayes filter \cite{chen2003bayesian}:
\begin{equation} \label{eqn:bayes_rule}
b_{t}(s')=\eta O(s',a_t,z_t)\sum\limits_{s\in S} T(s,a_t,s')b_{t-1}(s).
\end{equation}
The filter first considers the state-transition probabilities, $T(s,a_t,s')$, then applies the likelihood of the observation, $O(s',a_t,z_t)$, and finally, normalizes probabilities of states using a normalization constant $\eta$.
After update, the new belief $b_t$ becomes the entry point of the next planning cycle. 

\subsection{Computational Complexity} 
POMDP planning suffers from the well-known ``curse of dimensionality'' and ``curse of history'' \cite{Kaelbling_1998}.
The cost of building a full belief tree is $O(|A|^H|Z|^H)$, where $|A|$ and $|Z|$ are the size of the state and observation spaces, respectively, and $H$ is the planning horizon. \rvs{Online planning typically allows for a fixed amount of planning time, \eg, one second or below. Building full belief trees quickly becomes intractable when the action space size, observation space size, or the planning horizon increase.} 

State-of-the-art belief tree search algorithms, POMCP \cite{POMCP} and DESPOT \cite{DESPOT}, have made online POMDP planning practical. They have been successfully applied to real-world tasks such as autonomous driving \cite{bai2015intention,meghjani2019context}, clutter manipulation~\cite{xiao2019online}, multi-agent planning~\cite{walker2020framework}, \etc. 
Practical POMDP planning uses two core ideas: \textit{MC sampling} \rvs{and \textit{anytime heuristic search}.}
They sample the starting states and future outcomes of MC simulations, then condition belief tree search on the sampled states and observations to reduce the computational cost. \rvs{Further, the algorithms approximate the value of unsearched branches using heuristics, so that the tree search can be terminated at anytime and make approximate decisions.}

Specifically, DESPOT \cite{DESPOT} conditions belief tree search on a set of $K$ sampled \textit{scenarios}. 
It reduces the complexity of planning to $O(|A|^HK)$ while achieving near-optimality of decisions (see \secref{sec:search} for details).
HyP-DESPOT \cite{Hyp-despot} further scales up DESPOT through massive parallelization, achieving state-of-the-art performance on large-scale POMDP planning benchmarks. It has also been successfully applied to driving tasks in both off-road \cite{PORCA} and on-road \cite{SUMMIT} scenes.

Despite the progress, DESPOT and HyP-DESPOT may still struggle with \rvs{online} planning tasks with very long horizons or large action spaces, \rvs{producing highly sub-optimal decisions}. 
\algname addresses the challenge through learned heuristics. The learned policy network reduces the effective branching factor by guiding the search. The learned value network reduces the search to a shortened horizon $D$, beyond which \algname simply uses the learned values as approximations. 
\rvs{With heuristic values learned properly}, the cost of \rvs{optimal} online planning is further reduced to $O(|A|^DK)$. With $D\ll H$,  \algname  becomes exponentially more efficient.

\subsection{Integrating Planning and Learning}

Integrated planning and learning brings benefits from both  the power of explicit reasoning and the robustness of learning from data~\cite{AlphaGOZero}. \rvs{Recent advances in machine learning bring many new interesting opportunities in this direction.}

One approach is to develop \textit{differentiable planners}, \ie, impose a planning algorithm as the structure prior on the neural network (NN) architecture for learning, so that both the model and algorithm parameters are trained jointly end-to-end~\cite{VIN,qmdp-net,treeqn,mcts-net,upn,dpc}.
UPN \cite{upn} and DPC \cite{dpc} have implemented trajectory optimization and model predictive control using neural networks. 
For MDP/POMDP planning, VIN~\cite{VIN} and QMDP-Net~\cite{qmdp-net} encode the value iteration algorithm in an NN to solve navigation tasks. 
TreeQN embeds a fixed forward search tree into an NN~\cite{treeqn}.
MCTS-Net further performs dynamic tree search using learned tree search operators~\cite{mcts-net}. 
As expected, the learned networks face the same challenge of scalability as the underlying algorithm: value iteration works well only in low-dimensional discrete state spaces; searching a big tree using NN operators is not affordable in real-time planning.

Another approach is \textit{learning for planning}, \ie, injecting learned components into planning. 
A natural choice is to learn the dynamics and observation models and utilize them for planning or optimal control (model-based RL) \cite{ayoub2020model,planet,cavin}.
One may also learn sampling distributions \cite{liu2020learned}, local goals~\cite{waypoint-net}, or macro-actions \cite{MAGIC}, and use them to assist planning.
\rvs{We propose to learn heuristics for planning under uncertainty, in order to alleviate the exponential cost of online POMDP planning.} 
Our earlier work \cite{lets-drive} integrates POMDP planning with learning using two stages. The offline stage learns a policy and its value function from POMDP planning. The online stage uses the learned policy and values to guide online search. The algorithm has demonstrated success in driving among a crowd of pedestrians. However, its performance is limited, as the learner and the planner do not improve over time with additional experiences. 

This paper extends our earlier work \cite{lets-drive} in several aspects. 
First, \algname closes the planning-learning loop. 
The planner benefits from the learned policy and values for improved real-time planning performance; at the same time, it provides the data for learning the policy and value approximations. 
Second,  it provides theoretical guarantee on the near-optimality of its decisions despite the use of learned heuristics.
Finally, we apply \algname to a  more challenging driving setting in simulation, with dense, heterogeneous traffic and complex road structures. 

\section{Overview} \label{sec:overview}

The \algname algorithm consists of two interacting components: 
\begin{itemize}
    \item an online planner that plans robot actions  \figrefb{fig:overview}{a}, and
    \item a learner that learns policy and value approximations, represented as neural networks \figrefb{fig:overview}{b}.
\end{itemize}  
The planner and the learner run concurrently,  forming a closed loop between planning and learning. 

The planner is the actor. In each time step, it performs  belief tree search at the current belief,  using the learned policy and value networks \figrefb{fig:overview}{c} as heuristics to guide the search. It then chooses the action with the highest estimated value  for execution. At the same time, the planner collects data for learning. It forms a data tuple, consisting of the current belief, \rvs{supervision labels such as the planned action and the estimated value}, and the reward received from the environment after executing the action. It then adds the tuple to a data buffer to share with the learner.
Concurrently, the learner samples data from the shared buffer containing the planner's experiences and uses them to optimize the policy and value networks. 
In each  training iteration, the learner samples a fixed number of data points.
For self-supervised learning, it fits the policy network to the labeled actions and fits the value network to the labeled values \rvs{\figrefb{fig:overview}{d}}.
For reinforcement learning, it
optimizes the policy network using the reward from the environment \rvs{\figrefb{fig:overview}{e}}.

\begin{figure} [t]
\centering
    \centering
	\includegraphics[width=\columnwidth]{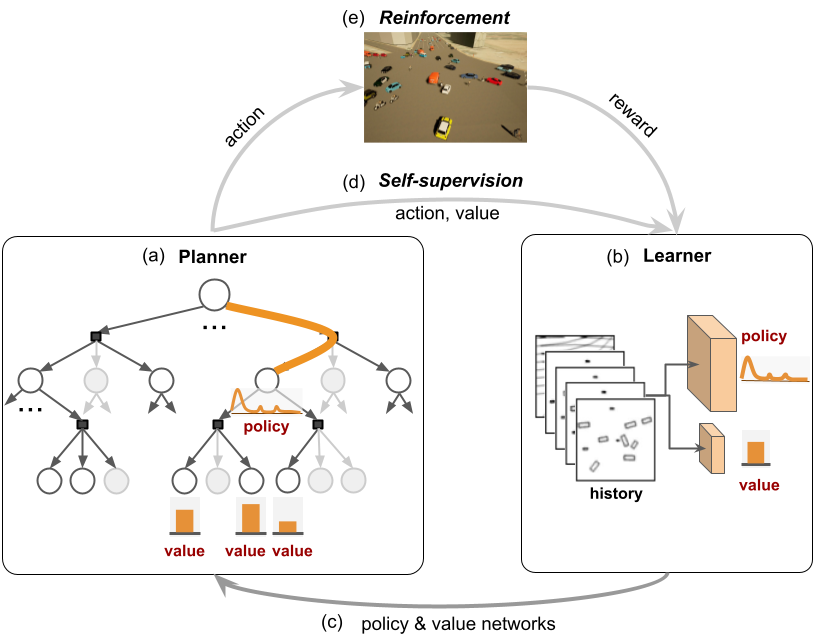}
	\caption{ \algname consists of an online planner and a learner, running concurrently. (a) Online POMDP planning. (b) Learning policy and value neural networks. (c) To speed up planning, the learner provides the planner policy and value approximations.  (d) The planner provides actions and values as data to the  self-supervised learner.  (e) The planner provides actions and the simulator provides rewards as data to the reinforcement learner.
	}
	\label{fig:overview}
\end{figure}

Effectively, \algname maintains two policies. 
The planner induces a policy implicitly. The learner represents a policy explicitly as a neural  network, learned from the planner's experiences.  
Both policies are useful. 
The learner policy directly maps state histories to actions. It is simple and fast. 
The planner policy performs guided belief tree search, on top of learned policy and value approximations. This improves the quality of action selection in complex situations. If the search tree depth is $0$, then the planner  and  learner policies become the same.

The planning-learning loop 
starts with randomly initialized policy and value networks, and improves both the planner policy  and the learner policy in synchrony through accumulated experiences.


In the following, we present a specific POMDP model for autonomous driving in a crowd as a representative task (\secref{sec:pomdp}). We then describe  the learning-guided planner  (\secref{sec:planning}) and the planning-informed learner (\secref{sec:learning}). 
The underlying idea of \algname is general and not specific to autonomous driving.

\section{POMDP for Driving in a Crowd} \label{sec:pomdp}
The \task task is to control an ego-vehicle to drive in dense, unregulated traffic, \eg, to cross an uncontrolled intersection as fast as possible. The ego-vehicle has to drive safely and conform to the lane network of the map, \ie, on lane directions and lane connectivity. The traffic is not properly regulated by traffic rules. Participants drive aggressively, \eg, can abruptly cut through the way of and overtake the ego-vehicle, forming a challenging dynamic environment.

The POMDP model primarily models the uncertainties in the \textit{intentions} and \textit{attentions} of \textit{exo-agents}. An exo-agent is a nearby traffic participant potentially interfering with the ego-vehicle within the planning horizon.
The intention of an exo-agent specifies which route on the urban map it intends to take, and its attention specifies whether it will actively avoid collision with others (attentive) or not (distracted). Our POMDP model is improved from the one described in \cite{SUMMIT} by allowing the ego-vehicle to plan both its driving path and longitudinal accelerations.
The new model also uses a factored reward function to facilitate value learning. 

\subsection{States and Beliefs}
Our state includes both continuous-valued physical states and discrete-valued hidden states of involved traffic participants:
\begin{itemize}
  \item Physical state of the ego-vehicle, $s_c=(x,y,\vec{v},\phi)$, including the position $(x,y)$, the velocity $\vec{v}$, and the heading direction $\phi$.
  \item Physical states of exo-agents, $\{s_i=(x,y,\vec{v},\phi)\}_{i\in I_{exo}}$, including the position $(x,y)$, the velocity $\vec{v}$, and the heading direction $\phi$ of each exo-agent. $I_{exo}$ defines the indices of exo-agents. Physical states can be detected from sensor inputs, but imperfectly.
  \item Hidden states of exo-agents, $\{\theta_i=(t_i, \mu_i)\}_{i\in I_{exo}}$, including the driver's attention $t_i$ (attentive / distracted) and the intended route $\mu_i$ of the $i$th traffic agent. The hidden states can not be detected by sensors and can only be inferred from history.
\end{itemize}
\rvs{We assume the set of exo-agents stays constant during planning, and update it before each online planning cycle.}
 
A \textit{belief} $b$ thus encodes a posterior distribution over 1) physical states of all agents, and 2) exo-agents' hidden variables (intentions and attentions). The belief is updated at every time step according to new observations using a Bayes filter \eqrefb{eqn:bayes_rule}.

\subsection{Actions}
An action of the ego-vehicle is the cross product of the lane-keeping/changing decision and the longitudinal acceleration. Each dimension contains three possible values: for lane decisions, $\{Left, Keep, Right\}$, and for accelerations, $\{Acc, Maintain, Dec\}$.
Candidate lanes are extracted from the lane network of the map.
A lane decision is executed using a pure-pursuit algorithm \cite{systembook} to track the center path of the selected lane.
Acceleration values for $Acc$ and $Dec$ are $3m/s^2$ and $-3m/s^2$, respectively. The maximum speed of the ego-vehicle is $6m/s$, from which it takes 2 seconds to reach a full stop.

\subsection{Observations and Observation Function}
The observation function captures the observability of state variables and the imperfection of sensing. For a given state $s=(s_c, \{s_i, \theta_i\}_{i\in I_{exo}})$, the observation $z$ is $(s_c, \{s_i\}_{i\in I_{exo}})$ with discretized values. 

\subsection{State Transitions \label{sec::transition}}
The state transition model simulates how traffic agents interact with each other. It inputs the state of all agents, $s=(s_c, \{s_i, \theta_i\}_{i\in I_{exo}})$, and the action of the robot, \act, and predicts the next state of all agents, $s'=(s_c', \{s_i', \theta_i\}_{i\in I_{exo}})$. Motion predictions are conditioned on the intention and attention of agents and constrained by kinematics.
The model assumes distracted agents blindly track their intended paths with the observed speeds; attentive agents follow an optimal local collision avoidance model \cite{luo2022gamma} to interact with others, while best following their intended paths. 
Kinematics of all vehicle-like agents are approximated using bicycle models \cite{systembook}; kinematics of pedestrians are modeled as holonomic. 
Finally, we perturb the displacements of all agents with Gaussian noises to model the stochasticity of human behaviors.

\subsection{Rewards} \label{sec:raw_reward}
The reward function is defined as follows.
When the vehicle collides with any exo-agent or obstacle, we assign a huge penalty, $R_\mathrm{col}=-1000\times(v^2 + 0.5)$, increasing quadratically with the colliding speed $v$, to enforce safety. For efficiency, we assign each time step a speed penalty $R_\mathrm{v} = 4(v-v_{max})/v_{max}$ to encourage driving at the maximum speed $v_{max}=6.0m/s$. We further impose a smoothness penalty $R_\mathrm{acc} = -0.1$ for each deceleration to penalize excessive speed changes, and a penalty of $R_\mathrm{change} = -4$ for each lane change to avoid jerky paths. 
The rewards are additive.

\subsection{Factoring Reward and Value Functions for Learning} \label{sec:factorization}
The above reward function effectively encodes the objective of safe, efficient, and smooth driving. However, it leads to a highly non-smooth value function---the magnitude of values drastically increases near collision events. To facilitate value learning, we have further decomposed the value function into two \textit{smooth} factors: a \textit{safe-driving factor} capturing efficiency rewards and smoothness penalties, and a \textit{collision factor} that captures collision risks and the corresponding penalties. Values of the two factors are computed using factored backup in the belief tree search. We learn the two factors independently and use the linear combination of them to recover the original value function. This is possible since the reward function is additive and the backup operator is linear. See Appendix \ref{sec:factored_reward} for details of the reward and value factorization.

\section{Learning-Guided Planning} \label{sec:planning}
To efficiently solve large-scale POMDP problems, \algname integrates online belief tree search with learned policy and value functions, using them to reduce the computational cost and improve real-time performance.
Our planner is built on top of HyP-DESPOT \cite{Hyp-despot}, a state-of-the-art online belief tree search algorithm. In \secref{sec:search}, we provide a brief summary of HyP-DESPOT. Then, we present our extensions over HyP-DESPOT in \secref{sec:use_priors}, and state our theoretical guarantee on the near-optimality of planning in \secref{sec:convergence}.

\subsection{HyP-DESPOT} \label{sec:search}

HyP-DESPOT samples a set of scenarios as representatives of the stochastic future.
Each scenario, $\scenario=(\sinit, \randnum_1,\randnum_2,...)$, contains an initial state \sinit sampled from the current belief and a sequence of random numbers, $\randnum_1,\randnum_2,...$, for determining the outcome of Monte Carlo simulations. Specifically, a simulation step,  $(s',z,r) = g(s,a,\randnum)$, samples a transited state $s'$, an observation $z$, and a reward $r$ using the POMDP model. Sampling is determinized by the input random seed $\randnum$, where $\randnum_i$ is used for the $i$th future time step. 
Each scenario thus corresponds to a deterministic belief tree of size $O(|A|^H)$, which only considers a single sampled observation as the outcome of each action.
HyP-DESPOT uses a collection of $K$ sampled scenarios to approximate the future, constructing a sparse belief tree of size $O(|A|^HK)$.
The root of the tree contains $K$ sampled initial states. Then, the tree recursively branches with all possible actions but only observations \textit{encountered under the sampled scenarios}.
Each node \node in the tree captures a subset of scenarios \scenariosetnode{\node} that visits the node, whose updated states approximate a future belief.

HyP-DESPOT performs anytime heuristic search to construct the belief tree. It maintains for each node an \textit{upper bound} and a \textit{lower bound} estimate of the node value, \ub{\node} and \lb{\node}, and uses them as tree search heuristics. In each iteration or each \textit{trial}, HyP-DESPOT starts from the root node \rootnode, \textit{traversing} a single \textit{exploration path} down to a leaf to expand the tree. 

At each node along the path, it selects the action branch with the highest optimistic outcomes:
\begin{eqnarray} \label{eqn:hypdespot_action_heuristics}
\optact&=&\argmax_{\act \in A}\uba{\node}{\act}
\end{eqnarray}
\vspace{-0.2cm}

\noindent
where $u(b,a)$ denotes the upper bound value to be achieved if applying \act at \node, computed from upper bounds of child nodes as in Eqn. \ref{eqn:bellman2}.
Moving down, the path traverses the observation branch with the maximum remaining uncertainty, which is measured using the gap between the upper and lower bound estimates. See \cite{Hyp-despot} for details of the observation selection heuristics.

When reaching a leaf node, the trial \textit{expands} the node using all possible next actions and sampled observations.
Afterward, the trial \textit{initializes} the upper and lower bound values of the new nodes. It performs Monte Carlo (MC) roll-outs to initialize lower bounds and uses explicit heuristic functions to initialize upper bounds (there generally exist ones that can be easily written down). Both are conditioned on the sampled scenarios. We refer to upper and lower bounds calculated in this way as the \textit{MC value estimates}. The initial MC estimates are denoted as $u_0(\node)$ and $l_0(\node)$.

The traversal continues until
further expansion is no longer beneficial. HyP-DESPOT thus ends the trial and immediately \textit{backs up} new information to the root, updating upper and lower bounds for belief nodes along the way using the Bellman's operator (Eqn. \ref{eqn:bellman1} and \ref{eqn:bellman2}).
After finishing the backup, a new trial starts.

HyP-DESPOT executes in an anytime fashion, terminating the search until a maximum planning time is reached or when the gap between the upper and lower
bounds at the root becomes sufficiently small, in which case it produces near-optimal actions. 
HyP-DESPOT further performs parallel tree search. Multiple threads execute concurrent trials to expand the tree collaboratively. 


\subsection{Incorporating Learned Heuristics} \label{sec:use_priors}
The \algname planner extends HyP-DESPOT, incorporating the learned policy and values in the heuristics
to perform guided belief tree search. It uses the policy network to guide forward traversal, and uses the value network to approximate long-term returns.
\figref{fig:overview}a briefly illustrates the guided belief tree search algorithm. 

The policy network is queried at each node along the exploration path to provide \textit{prior probabilities} over actions \figrefb{fig:overview}{a-policy}. The probabilities are used to bias action selection. Specifically, a trial visiting a node $b$ selects an action branch to traverse using a UCB-like heuristics:
\begin{equation}
a^*=\argmax_{a\in A} \left \{ u(b,a) + c \;\pi_{\theta}(a | x_b)\sqrt{\frac{ N(b)}{N(b,a)+1}} \right \}. 
\label{Eqn:select_ub}
\end{equation}
\noindent
The improvement of Eqn. \eqref{Eqn:select_ub} over Eqn. \eqref{eqn:hypdespot_action_heuristics} is the additional exploration bonus weighted by a constant factor $c$. $\pi_{\theta}(\cdot| x_b)$ denotes the prior probabilities output by the policy network $\pi_{\theta}$ at the history state $x_b$, a 4-step history at $b$ encoded as images (see \secref{sec:nn_architecture}). 
It prioritizes actions suggested by the learned policy.
The bonus further depends on the visitation count of $b$, \ie, the number of times \node has been visited by previous trials, denoted as \nvisit{\node}, and the visitation count of its child action branch, \nvisita{\node}{\act}. 
This encourages exploration.
When visiting a node the first few times, upper bound estimates of actions, $u(b,\cdot)$, are often uninformative. The action selection is therefore strongly biased by the learned policy, with a desirable level of exploration ensured by the visitation count term. After sufficient search, the difference of upper bounds gradually dominates the heuristics, making it behave more similar to Eqn. \eqref{eqn:hypdespot_action_heuristics}. The observation selection heuristics remain the same as HyP-DESPOT.

The value network is queried at each leaf node to 
an initial value estimate of the node \figrefb{fig:overview}{a-value}:
\begin{equation} \label{eqn:letsdrive_value}
\hat{v}_0(\node)=\min\bigl(\max\bigl(l_0(\node), v_{\theta'}(x_\node)\bigr), u_0(\node)\bigr),
\end{equation}
where \node is a new leaf node, and $v_{\theta'}(x_b)$ is the \textit{prior value} predicted by the value network $v_{\theta'}$ at the history state $x_b$ (see \secref{sec:nn_architecture}). The prior value has been learned from past experience, providing accurate value estimates that can otherwise only be acquired by searching the corresponding sub-tree sufficiently.
Eqn. \eqref{eqn:letsdrive_value} further performs \textit{value clipping} to regulate the prior values. For a leaf node \node, it clips the prior value, $v_{\theta'}(x_b)$, using the initial MC bounds, $l_0(b)$ and $u_0(b)$, to produce an initial learned value, $\hat{v}_0(b)$.
The value clipping guarantees the correctness of learned values, ensuring a relationship of $l_0(b) \le \hat{v}_0(b) \le u_0(b)$, which helps maintain theoretical guarantees of planning. See theoretical details in \secref{sec:convergence}.

During backup, we update both the learned value and the MC value estimates of belief nodes, using the Bellman's operator (Eqn. \ref{eqn:bellman1} and \ref{eqn:bellman2}).
Since the operator only applies linear operations and maximization, it guarantees the same relationship, $l(b) \le \hat{v}(b) \le u(b)$, to hold for all belief nodes throughout the tree search.

When the tree search terminates, \algname reports the action with the best \textit{learned value} at the root, which provides the best long-term outcome, estimated by the search tree.
%

\begin{figure*} [!t]
\centering
\includegraphics[width=0.9\textwidth]{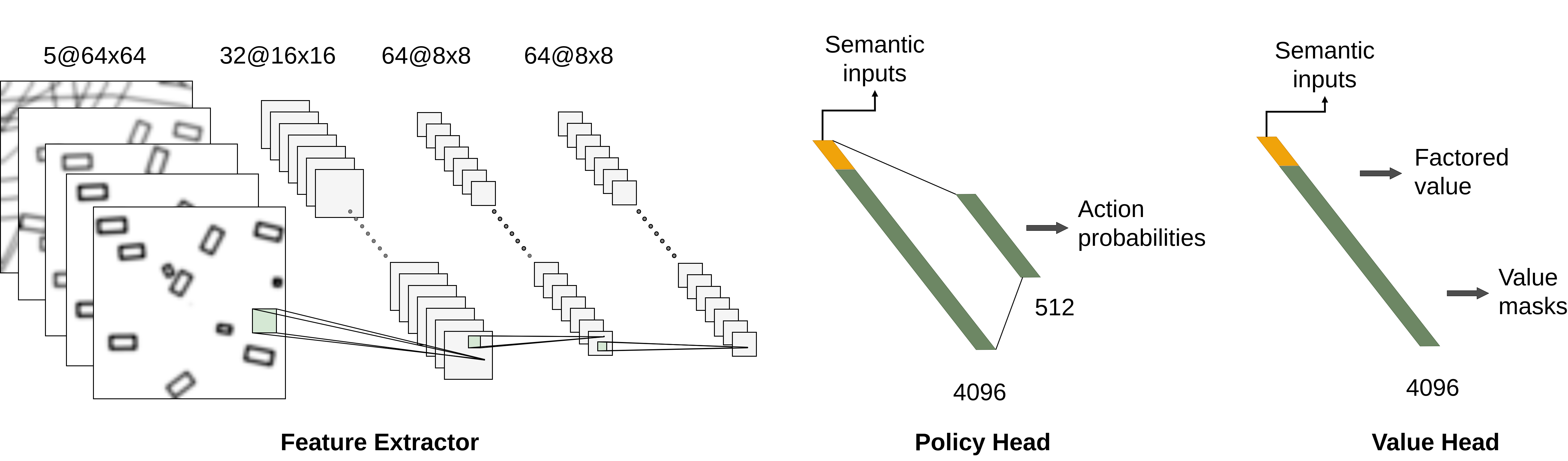}
	\caption{Neural network architectures of the policy and value networks.}
	\label{fig:architectures}
\end{figure*}

\subsection{Performance Guarantee} \label{sec:convergence}
The following theorem analyzes the \textit{convergence} of the uncertainty gap at the root \rootnode, measured by the difference between the upper and lower bound estimates, $\gapsymbol(\rootnode) = u(\rootnode)-l(\rootnode)$, and discusses the \textit{regret bound} of the reported optimal policy $\optdespotpolicy$ at convergence: 
\begin{theorem} \label{convergence}
    The proposed belief tree search algorithm using learned heuristics is probabilistically and asymptotically optimal. Suppose that the maximum planning time is unbounded. The uncertainty gap at the root, $\gapsymbol(\rootnode)$, will converge to zero in finite time. Let $\optdespotpolicy$ denote the policy tree reported by the algorithm at convergence. Given any constant $\tau \in (0,1)$, the following relationship holds for the value of $\optdespotpolicy$ and the value of the true optimal policy $\optpolicy$, with probability at least $1-\tau$:
    \begin{equation}\label{eqn:despot_error}
    V_{\optdespotpolicy}(\rootnode) \ge  V_{\optpolicy}(\rootnode) - \regretsymbol_{\optpolicy, \tau}(K).
    \end{equation}
The approximation error $\regretsymbol_{\optpolicy,\tau}(K)$ is the same regret bound stated in Theorem 3.2 of \cite{DESPOT}, which approaches zero when $K\rightarrow \infty$, at a rate of $O(\frac{1}{\sqrt{K}})$, where $K$ is the number of sampled scenarios.
\end{theorem}

\begin{proof}
Convergence of the search is guaranteed by \textit{optimistic trails}, exploration paths that uses the unbiased heuristics \eqrefb{eqn:hypdespot_action_heuristics}, which is brought forward from HyP-DESPOT \cite{Hyp-despot} and launched periodically in our planner.
As shown in Theorem 1 of \cite{Hyp-despot}, optimistic trials always guarantee the uncertainty gap at the tree root, $\gapsymbol(\rootnode) = u(\rootnode)-l(\rootnode)$, to monotonically decrease and converge to zero with a finite number of trials, and the guarantee holds regardless of what exploration mechanism is deployed in other trials, in our case, Eqn. \eqref{Eqn:select_ub}. 
This means the proposed algorithm always converges in finite time. 
Since the learned value at the root node \rootnode is bounded between $l(\rootnode)$ and $u(\rootnode)$, it will converge to the optimal value under the sampled scenarios. Namely, our planner will report the same optimal policy as HyP-DESPOT when both algorithms converge. 

Further, as shown in Theorem 3.2 of \cite{DESPOT}, given any constant $\tau \in (0,1)$, the regret induced by the optimal HyP-DESPOT policy, $\optdespotpolicy$, which  considers only  the sampled scenarios, with respect to the true optimal policy, $\optpolicy$, that considers all possible scenarios, is bounded by $\regretsymbol_{\optpolicy,\tau}(K)$ with probability at least $1-\tau$. 
For a given POMDP problem, the regret bound, $\regretsymbol_{\optpolicy,\tau}(K)$, is determined by $K$, the number of sampled scenarios. When increasing $K$ to infinity, the bound converges to zero at an asymptotic rate of $O(\frac{1}{\sqrt{K}})$. The regret bound also decreases with the size of the optimal policy tree $\optpolicy$, meaning that if a simple near-optimal policy exists, the value approximation will be particularly tight. We refer readers to Theorem 3.2 of \cite{DESPOT} for the detailed expression of Eqn. \eqref{eqn:despot_error} and a rigorous proof of the bound.

Finally, since our planner reports the same policy as HyP-DESPOT at convergence, it provides the same regret bound. This concludes our proof that our planner is probabilistically and asymptotically optimal. 
\end{proof}

\section{Planning-Informed Learning}  \label{sec:learning}
The learner in \algname uses the planner's experiences to update the policy network and the value network. 
This section will discuss three possible designs of learners---\textit{open-loop self-supervision}, \textit{closed-loop self-supervision}, and \textit{closed-loop reinforcement}. The learners share the same planner counterpart, thus also correspond to three \algname variants.
In the following, we will first present the neural network architectures, and the collection of learning experience, then present the three learner variants' core ideas and algorithmic details.

\subsection{Policy and Value Networks} \label{sec:nn_architecture}

The architectures of the policy and value networks are shown in \figref{fig:architectures} and described below.
Input to the policy and value networks are top-down rasterized images encoding the state history $x_b$ at a belief $b$. The input consists of 5 channels. Channel $1-4$ encode the geometry of traffic agents at the current and three past frames; The $5th$ channel encodes the lane graph of the urban map drawn as a set of poly-lines. All images are registered to the local view of the ego-vehicle for the corresponding time step. They are initially rendered as $1024\times1024$ images and down-sampled to $64\times64$ using Gaussian pyramids \cite{adelson1984pyramid} before inputting to the neural networks.

The policy and value networks use a feature extractor similar to that in \cite{dqn}. The input images are processed by three convolutional layers: an input layer with 32 $8\times 8$ kernels with stride $4$ and no padding; a middle layer with 64 $4\times 4$ kernels with stride $2$ and no padding; and the last layer with 64 $3\times 3$ kernels with stride $1$ and no padding. The extractor outputs 64 $8\times 8$ images as hidden features. These features are flattened and concatenated with the semantic inputs, \ie, velocities of the ego-vehicle in the past four frames, and fed to the heads.

Our policy network only has one categorical head to output the distribution over nine possible lane-decision / acceleration combinations. The policy head has two fully-connected (FC) layers mapping from the raw feature vector of length $4096$ to an intermediate feature vector of length $512$, then to $9$ action probabilities. The value network, instead, has two heads corresponding to the factored value function (Appendix \ref{sec:factored_reward}). They include a mask head to output two binary masks for the safe-driving and collision value factors, and a value head to predict the non-zero numbers for the value factors. Both heads have a single FC layer directly mapping the raw features to factored predictions, which are combined to recover the actual value prediction.

\subsection{Data Collection}\label{sec:mixed_experience}
Data for learning are collected by the planner, through acting in the environment.
In each episode, the actor records a trajectory. 
Each data point along the trajectory is represented as $(b,a,r,a^*, v^*)$, where $b$ is the belief at the corresponding time step, $a$ is the executed action, and $r$ is the reward fed back by the environment after executing $a$; the data point also records the planner's estimation of the optimal action, $a^*$, and the optimal value, $v^*$, at $b$.
The action-value labels enable self-supervised learning (Sections \ref{sec:ol_learner} and \ref{sec:cl_spv_learner}). 
The rewards enable reinforcement learning \secrefb{sec:cl_rnf_learner}.
\rvs{The planner actor executes different $a$ for different learners. For self-supervised learning, we directly execute the optimal action ($a=a^*$); for reinforcement learning, we use a combination of exploitative actors ($a=a^*$), explorative actors (sample $a$ according to estimated action-values), and on-policy actors (execute the learner policy) to collect experience.}
Collected trajectories are processed into a pool of data points, either stored in an offline dataset or fed to a fixed-capacity replay buffer, for offline and online learners, respectively. 
Multiple actors can execute asynchronously in separate simulator instances, to collaboratively collect data.

\subsection{Learners}
Now we introduce the learner that uses experiences from the planner to optimize the policy and value networks. We propose the following learner variants, covering both open-loop and close-loop integration of planning and learning, and leveraging both self-supervised and reinforcement learning:

\subsubsection{Open-loop self-supervised learning (\olimitation)} \label{sec:ol_learner}
In \olimitation, the integration of planning and learning happens in two phases: offline supervised learning and online guided planning.
In the offline phase, \olimitation learns from a fixed planning expert;
In the online phase, it plans with learned heuristics.
No further data is fed back to the learner during the online stage. The planner and the learner are thus integrated in an ``open-loop''. This algorithm replicates the idea of our earlier work \cite{lets-drive} using the new planner \rvs{with value clipping} \secrefb{sec:planning} and the new \rvs{POMDP and neural network} models (\secref{sec:pomdp} and \ref{sec:nn_architecture}).

Specifically, \olimitation uses HyP-DESPOT \cite{Hyp-despot} as the actor to collect an offline dataset. 
Then, the learner trains the neural networks by sampling the dataset. For each sampled belief $b$, it fits the policy network to the action label, $\optact$, and fits the value network to the value label, $v^*$, both provided by HyP-DESPOT. It uses cross-entropy loss (CEL) for policy predictions and mean-square errors (MSE) for value predictions. 
To facilitate value learning, we have further decomposed the value loss into safe-driving and collision factors following the factorization of the planner's value function (\secref{sec:factorization}). Details of the loss functions are explained in Appendix~\ref{sec:spv_loss}.

At execution time, \olimitation performs guided belief tree search to synthesize real-time driving policies (\secref{sec:use_priors}). 
\olimitation thus benefits from both local planning and global learning. 
However, the limitation is that it cannot leverage new data generated by the stronger, guided planner.

\subsubsection{Closed-loop self-supervised learning (\climitation)} \label{sec:cl_spv_learner}
\climitation improves over \olimitation by letting the learner receive online experiences from the \textit{guided} planner and constantly feed updated heuristics back to the planner, thus closing the planning-learning loop as shown in \figref{fig:overview}.

\climitation uses the guided planner \secrefb{sec:planning} as the actor to collect data for learning.
The planner uses the latest policy and value network as heuristics.
In each episode, the planner collects a trajectory and feeds it to a fixed-capacity replay buffer. 
In the meantime, the learner repeatedly samples data from the replay buffer. 
For each sampled belief $b$, it fits the policy network to the action label, $a^*$, and the value network to the value label, $v^*$, both provided by the guided planner \figrefb{fig:overview}{d}. Policy and value learning also uses CEL and MSE losses, respectively, where the value loss is factorized. After every few updates, it feeds the new heuristics back to the planner through a shared buffer. \climitation learns from scratch, starting from randomly initialized policy and value networks and an empty replay buffer. 
Stable training is achieved with the help of entropy regularization\rvs{, which enforces a desired level of entropy for the learned policies.} Details of the loss function and entropy regularization are explained in Appendix~\ref{sec:spv_loss}.

\climitation is essentially a form of \textit{self-supervised} learning: the planner provides labels to train its own sub-components (the heuristics). Sample efficiency is achieved by using structured rewards (compiled as values) from the planner as learning signals, instead of unstructured (raw) rewards from the environment. 

\climitation can also be viewed as generalized policy iteration \cite{DecisionBook}: the belief tree search performs policy improvement over the current policy; the learner then updates itself to fit the improved policy. By iterating these two steps, the planner and the learner can together converge to optimal policies defined w.r.t the POMDP model.

The POMDP model is, however, an imperfect approximation to the actual environment. 
Since \climitation~plans using the model and learns from policies and values generated with the model, it can be sensitive to model errors (even though we observe it working well in practice).

\subsubsection{Closed-loop reinforcement learning (\clreinforce)} \label{sec:cl_rnf_learner}
\clreinforce is thus proposed to hedge against model errors.
\clreinforce additionally uses policy gradient \cite{A3C,DDPG,SAC} to let the policy network receive and learn from reward feedback from the \textit{actual environment}. By doing so, the learner policy is optimized w.r.t. the true environment dynamics. 

\clreinforce shares the same closed-loop architecture and value learner as \climitation. 
Differently, the policy learner is not supervised by the planner's actions. Instead, it learns from the rewards fed back by the environment \figrefb{fig:overview}{e}.
At each sampled belief $b$, \clreinforce reinforces the learner policy by estimating its expected value from raw reward signals along the trajectory, and differentiating the value to compute gradients for updating the learner policy (``policy gradient''). \clreinforce thus optimizes the learner policy for its \textit{own} expected value. The policy is thus unaffected by the imperfection of planning models.  

\clreinforce essentially uses the planner as an \textit{exploration policy} for reinforcement learning. However, the explored trajectories are \textit{off-policy}, i.e., not sampled from the distribution induced by the learner policy, but from the distribution induced by the guided planner. Such trajectories lead to biased value estimates and policy gradients for the learner policy if not properly corrected.
Thus, we build our learner on top of a popular off-policy policy gradient algorithm, soft actor-critic (SAC) \cite{SAC-Discrete}, to correctly train the learner policy using the planner's experience. \rvs{Entropy regularization is also applied here to assist training.} Details of our SAC implementation are presented in Appendix \ref{sec:rnf_loss}.

\section{Experiments} \label{sec:results}
\begin{figure*} [!t]
    \centering
    \begin{tabular}{ccc}
    \hspace*{-0.5cm}
    \includegraphics[width=0.34\textwidth]{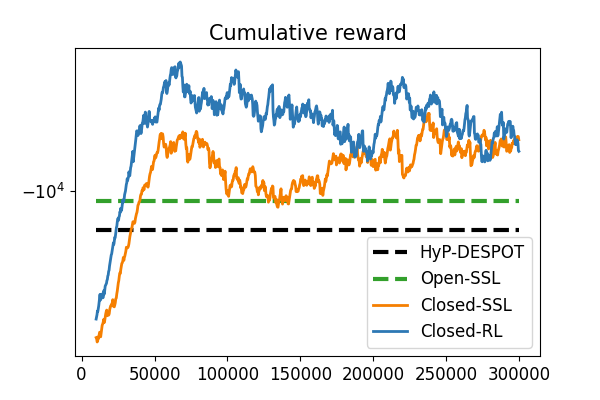} &
    \hspace*{-0.65cm}
    \includegraphics[width=0.34\textwidth]{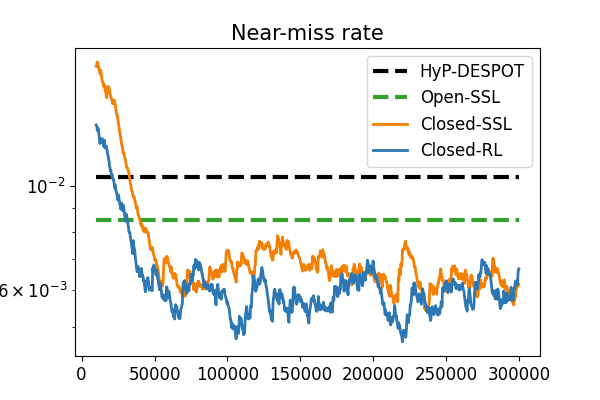} &
    \hspace*{-0.65cm}
    \includegraphics[width=0.34\textwidth]{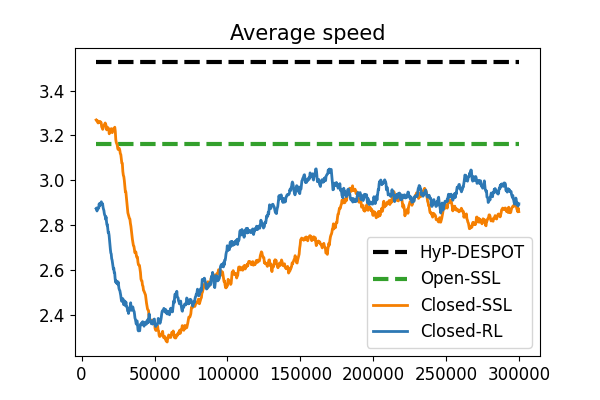}
    \vspace*{-0.2cm}
    \\
    (a)&(b)&(c)
    \end{tabular}
    \caption{Performance of planner policies in \algname compared with online POMDP planning using HyP-DESPOT. All \algname variants achieve significant improvements over POMDP planning. \climitation and \clreinforce achieve the best sample efficiency and asymptotic performance.
    }
    \label{fig:planner_exp}
\end{figure*}

\begin{figure*} [!t]
    \centering
    \begin{tabular}{ccc}
    \hspace*{-0.5cm}
    \includegraphics[width=0.34\textwidth]{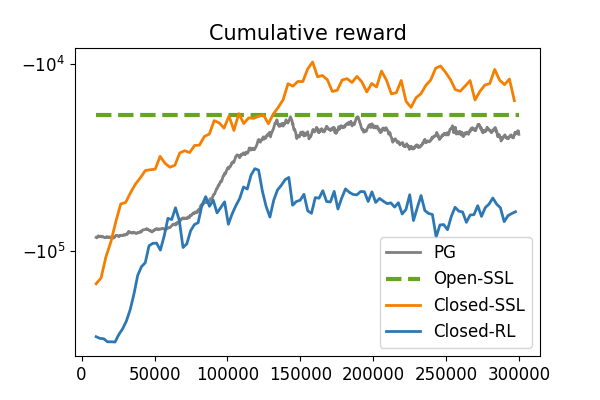} &
    \hspace*{-0.65cm}
    \includegraphics[width=0.34\textwidth]{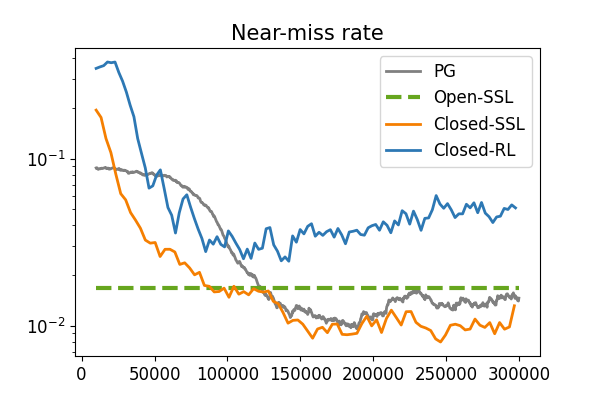} &
    \hspace*{-0.65cm}
    \includegraphics[width=0.34\textwidth]{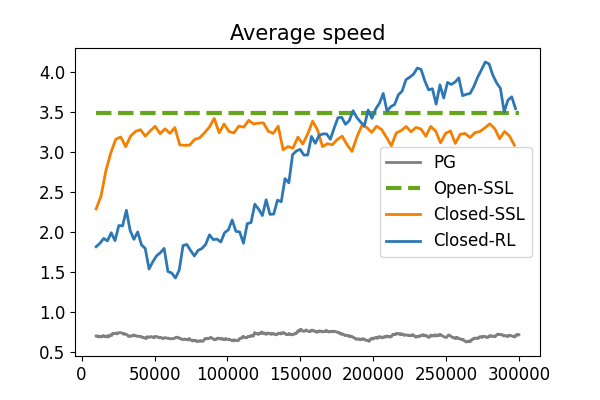} 
    \vspace*{-0.2cm}\\
    (a)&(b)&(c)
    \end{tabular}
    \caption{Performance of learner policies in \algname, compared with policy gradient (PG) and imitation learning (\olimitation). Closed-loop self-supervised learning (\climitation)  produces  the most effective learner policy.}
    \label{fig:learner_exp}
\end{figure*}

In the experiments, we compare \algname with online POMDP planning, reinforcement learning, as well as open-loop integration of planning and learning. We also provide tests on the scalability and generalization of \algname, and analyzed the importance of new algorithmic components. We tested all three variants of \algname. Among them, \textit{\climitation} and \textit{\clreinforce} are our proposed approaches, and \textit{\olimitation} serves as an open-loop integration baseline. For planning and learning baselines, we use HyP-DESPOT to calibrate the capability of existing POMDP planning tools, and use imitation learning (the policy learner in \olimitation) and policy gradient using SAC \cite{SAC-Discrete} (labeled as \textit{\sac}) to calibrate the capability of stand-alone policy learning. By comparing \algname with the existing algorithms that it is developed on top of, we perform controlled experiments to clearly show the benefit of integrating planning and learning for tackling short planning time and limited data.

Our results show that the integration of planning and learning enables \algname to largely advance the capability of both, greatly improving the scalability of online planning and the efficiency of learning. Closed-loop planning and learning further improves the sample efficiency and asymptotic performance by a large margin.
When using self-supervised learning, \algname produces
the strongest learner policies; when additionally using reinforcement learning, \algname achieves the best integrated performance. Value clipping applied in the search not only ensures theoretical guarantees, but also improves the practical performance of \algname. After training, our planner can successfully drive a vehicle through dense urban crowds with sophisticated combinations of accelerations and maneuvers, and generalize to significantly different environments. See the example driving clips in the accompanying video or via this link: \href{www.dropbox.com/s/n8t5dxo1i295smy/tro-lets-drive.mp4?dl=0}{www.dropbox.com/s/n8t5dxo1i295smy/tro-lets-drive.mp4?dl=0}. 

\subsection{Experimental Setup}

We analyze our approach in SUMMIT \cite{SUMMIT}, a real-time simulator for massive mixed urban traffic. Given any worldwide location supported by the \href{www.openstreetmap.org}{OpenStreetMap} \cite{OSM}, SUMMIT automatically generates dense traffic on the map. It controls exo-agents using GAMMA~\cite{luo2022gamma}, a recent traffic motion model validated on multiple real-world datasets. The model uses velocity-obstacle-based optimization to perform local collision avoidance. It is efficient, able to support real-time simulation of many traffic agents.
We train \algname using random crowds at the Meskel-Square intersection at Addis Ababa, Ethiopia (\figref{fig:meskel}) simulated in SUMMIT. Each instance of urban crowd contains 110 active traffic agents, including trucks, buses, cars, motorcycles, pedestrians, \etc. 
An episode of experience consists of a few minutes of continuous driving and ends when the vehicle exits the range of the map.

We measure the performance of planning and learning algorithms using the average cumulative reward achieved over episodes. We further measure the driving safety and efficiency to provide a detailed view of performance. Driving safety is measured as the near-miss rate over all time steps. A near-miss is a close encounter event when the estimated time-to-collision is shorter than a threshold, here set as $0.33s$, which is a more robust hazard indicator than collisions \cite{nearmiss}. Driving efficiency is characterized by the average driving speed, measured in $m/s$.

We use the same planning setup for all planners and the same learning setup for all online learners.
A POMDP state tracks a maximum of 20 exo-agents within 50 meters of the ego-vehicle, corresponding to the rough attention range of human drivers in dense traffic.
Planners are allowed $0.3$ seconds of maximum planning time at each step and execute at a rate of $3Hz$, which roughly reflects the human response time.
All learners use $3\times 10^5$ data points. Training of \climitation, \clreinforce, and \sac starts with an empty replay buffer, and ends after receiving $3\times 10^5$ unique data points. Since the time for collecting a data point in real-time simulation is fixed, the setup also leads to the same learning time, which is approximately 20 hours when using three concurrent actors and one learner on a single machine\footnote{
The machine is a server with 4 RTX 2080 GPUs, an Intel(R) Core(TM) i7-8750H CPU, and 256G RAM. Each of the actors and the learner uses one GPU to query or train neural networks.}. The \olimitation~baseline consumes an offline dataset of size $3\times 10^5$ and is trained till convergence. The reinforcement learner in \clreinforce uses the same network architectures as \sac. 
Policy and value networks in all \algname variants share the same network architectures.

\subsection{Planner Policies} \label{sec:exp_planners}

\begin{figure*} [!t]
    \centering
    \begin{tabular}{ccc}
    \hspace*{-0.4cm}
    \includegraphics[width=0.3\textwidth]{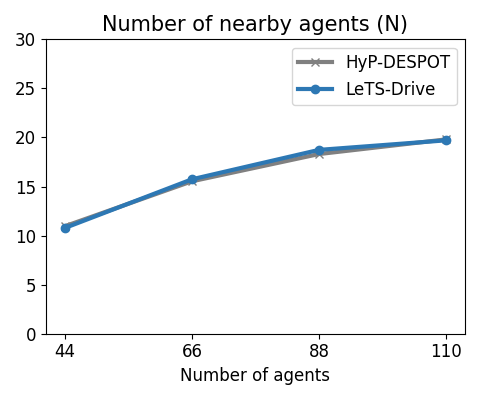} & 
    \hspace*{-0.65cm}
    \includegraphics[width=0.3\textwidth]{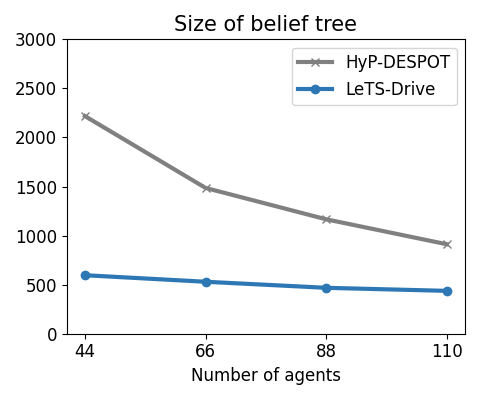} & 
    \hspace*{-0.65cm}
    \includegraphics[width=0.3\textwidth]{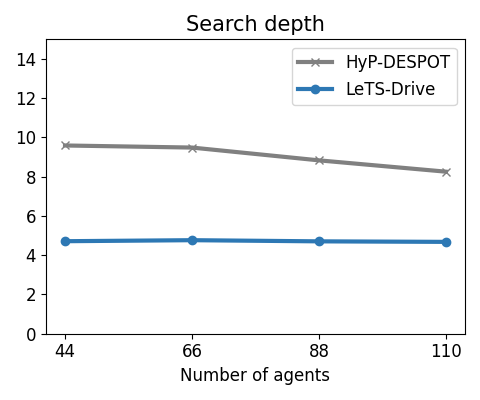}
    \vspace*{-0.2cm} \\
    (a)&(b)&(c)\\
    \hspace*{-0.4cm}
    \includegraphics[width=0.3\textwidth]{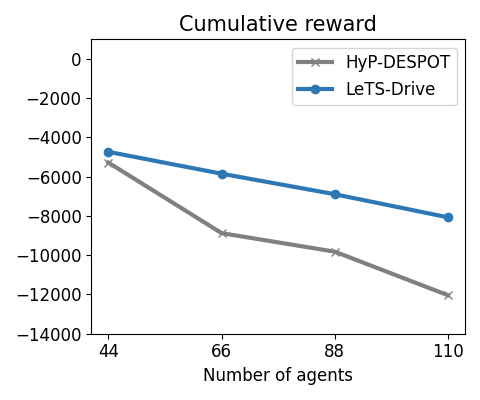} &
    \hspace*{-0.65cm}
    \includegraphics[width=0.3\textwidth]{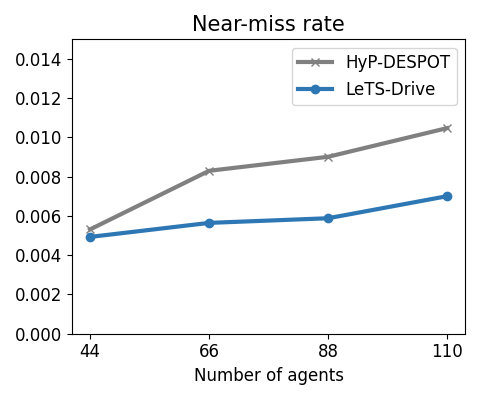} &
    \hspace*{-0.65cm}
    \includegraphics[width=0.3\textwidth]{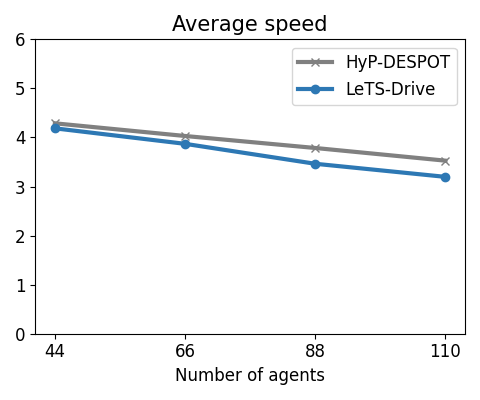} 
    \vspace*{-0.2cm}\\
    (d)&(e)&(f)
    \end{tabular}
   
    \caption{Scalability of \algname planner (\clreinforce) with  increasing number of agents in the crowd, compared with standard POMDP planning using HyP-DESPOT).}
    \label{fig:scalabilty_exp}
\end{figure*}

\figref{fig:planner_exp} shows the learning curves of the \textit{planner policies} of \algname and the performance of online POMDP planning using HyP-DESPOT. 
The learning curves are generated by periodically evaluating the planners in SUMMIT throughout training, averaged over five random seeds.
The main observations are as follows.

The integration with learned heuristics immediately brings significant performance gain over HyP-DESPOT, even when using the open-loop architecture (\olimitation). The resulting planner policy conducts more cautious driving, leading to fewer near-misses.

By closing the planning-learning loop using self-supervision (\climitation), \algname achieves superior sample efficiency, outperforming open-loop integration with around one-tenth of data, 
and achieving much higher asymptotic performance. The resulting planner policy further reduces the near-miss rate by a large margin.

Switching to the reinforcement learner (\clreinforce) further improves the integrated performance, as the algorithm additionally receives feedback from the actual environment. \clreinforce~brings the best sample efficiency, quickly achieving the highest rewards among all planners during training.

We have observed similar learning patterns from the \climitation and \clreinforce planners. Both of them first learn to reduce the near-miss rate by lowering the driving speed. Then, they gradually increase the driving speed with the near-miss rate maintained low. Both training curves have converged after receiving $1.5\times10^5\sim2\times10^5$ data points. At convergence, they deliver a similar level of planning performance. 

\subsection{Learner Policies} \label{sec:exp_learners}

\figref{fig:learner_exp} shows the learning curves of the \textit{learner policies} in \algname and stand-alone policy learning approaches. The curves are generated by periodically evaluating the policy networks in SUMMIT throughout training. 

We observe that policy gradient learners, which  do not perform  explicit reasoning, struggle to learn an effective policy for \task given the limited amount of data. This is because the task conveys three distinct local-optima behaviors: defensive driving, aggressive driving, and smart collision avoidance (desired). 
Policy gradient (\sac) acquires very conservative driving behaviors, primarily learning to reduce the near-miss rate, which is safe but inefficient. 
The learner policy of \clreinforce acquires overly-aggressive behaviors, mostly learning to increase the driving speed, which is efficient but unsafe. 
The behavioral difference  results primarily from  different experiences.
The \sac policy only learns from its own driving experiences. It  hardly foresees the benefit of gathering speed due to  frequent collision penalties.
The \clreinforce policy learns from the planner's experiences. It is incentivized to gather speed by the positive rewards received. However, with limited data, \clreinforce struggles to learn  collision avoidance sufficiently well.

In comparison, self-supervised learning (\climitation) produces smart driving policies with both low near-miss rates and desirable driving efficiency. The final learner policy has matched the performance of HyP-DESPOT, showing the effectiveness of self-supervision in policy learning.

\begin{table} [!t]
    \centering
    \caption{Generalization of trained \algname planner policies over new crowd distributions in the training map. The first and second columns show the improvement on the average cumulative reward compared to the learner policy and POMDP planning using HyP-DESPOT, respectively.}
    \label{tab:driving_results}
    \resizebox{\columnwidth}{!}{%
    \begin{tabular}{lrrrr}
    \toprule
    & \begin{tabular}[l]{@{}l@{}}Reward \\ w.r.t.  learner \\($\times 10^3$) \end{tabular} 
    & \begin{tabular}[l]{@{}l@{}}Reward \\ w.r.t POMDP \\ ($\times 10^3$) \end{tabular} 
    & \begin{tabular}[l]{@{}l@{}}Near-miss rate\\~ \\ ~ 
    \end{tabular}
    & \begin{tabular}[l]{@{}l@{}}Average speed\\~ \\ ~ 
    \end{tabular}\\
    \midrule
    HyP-DESPOT & - & 0.00 & 0.0100 & 3.53$\pm$0.000 \\
    \olimitation
    & +8.34 & +1.54 & 0.0085 & 3.16$\pm$0.000 \\
    \climitation
    & +5.19 & \textbf{+4.03} & \textbf{0.0057} & 2.74$\pm$0.003  \\
    \clreinforce
    & \textbf{+35.97} & \textbf{+3.95} & 0.0066 & 3.00$\pm$0.005 \\
    \bottomrule
    \end{tabular}}
\end{table}

\begin{table} [!t]
    \centering
    \caption{Generalization of trained \algname planner policies over novel maps. The first and second columns show the improvement on the average cumulative reward compared to the learner policy and POMDP planning using HyP-DESPOT. 
    }
    \label{tab:driving_results_test_map}
        \resizebox{\columnwidth}{!}{%
    \begin{tabular}{lrrrr}
        \toprule
        & \begin{tabular}[l]{@{}l@{}}Reward \\ w.r.t. learner\\ ($\times 10^3$)\end{tabular} 
        & \begin{tabular}[l]{@{}l@{}}Reward \\ w.r.t. POMDP \\ ($\times 10^3$) \end{tabular} 
        & \begin{tabular}[l]{@{}l@{}}Near-miss rate\\~ \\ ~ 
        \end{tabular}
        & \begin{tabular}[l]{@{}l@{}}Average speed\\~ \\ ~ 
        \end{tabular} \\
        \midrule
        HyP-DESPOT & \rvs{-} & 0.00 & 0.0081 & 3.57 $\pm$ 0.026 \\
        \olimitation
        & \rvs{+2.74} & +2.65 & 0.0057 & 3.09 $\pm$ 0.000 \\
        \climitation
        & \rvs{+13.2} & +3.81 & \textbf{0.0049} & 2.74 $\pm$ 0.016 \\ 
        \clreinforce
        & \rvs{\textbf{+46.4}} & \textbf{+4.15} & 0.0052 & 3.03 $\pm$ 0.017 \\
        \hdashline
        \rvs{\begin{tabular}[l]{@{}l@{}}\clreinforce \\ (Retrain) \end{tabular}} 
        & \rvs{+44.9} & \rvs{+4.20} & \rvs{0.0049} & \rvs{3.20} $\pm$ \rvs{0.020} \\
        \bottomrule
    \end{tabular}
    }
\end{table}

\subsection{Scalability}

We further provide in \figref{fig:scalabilty_exp} a scalability test for the \rvs{\algname(\clreinforce)} planner and compare it with the scalability of HyP-DESPOT. In the test, we gradually increase the number of agents in the Meskel Square from 44 to 110, to construct planning problems of different scales and complexities. The scalability of planners are evaluated using their capability of handling the growing problem scale in real-time. We observe \algname consistently outperforming HyP-DESPOT by searching smaller and shallower trees, when using the same planning time of $0.3s$. The performance gain increases with the problem scale.

The problem scale of \task is determined by the number of nearby agents a planner considers in a POMDP state, here denoted as $N$. Increasing $N$ leads to an exponential growth of the state and observation spaces and a quadratic growth of the complexity of the transition function. $N$ increases with the crowd size, as more agents fall within the 50-meter range of the ego-vehicle. 
\figref{fig:scalabilty_exp}{a} shows the number that planners effectively considered in the experiments. 

The growth of the problem scale leads to quickly decayed real-time performance of HyP-DESPOT, as shown by the declined reward \figrefb{fig:scalabilty_exp}{d}, increased near-miss rate \figrefb{fig:scalabilty_exp}{e}, and sacrificed driving speed \figrefb{fig:scalabilty_exp}{f}.
In comparison, \algname always generates better policies by searching much smaller and shallower trees. When the problem scale grows, the tree size and search depth of \algname remain almost unaffected \figrefb{fig:scalabilty_exp}{bc}. However, the benefit of integrating planning and learning increases. \algname consistently achieves higher cumulative rewards \figrefb{fig:scalabilty_exp}{d} and lower near-miss rates \figrefb{fig:scalabilty_exp}{e}, with a marginal compromise on the driving speed \figrefb{fig:scalabilty_exp}{f}.
The denser the scene is, the more performance gain \algname brings, showing improved scalability of planning. 

\subsection{Generalization}

\begin{figure} [!t]
    \centering
    \begin{tabular}{cc}
    \centering
        \Fbox{\includegraphics[height=0.22\textwidth]{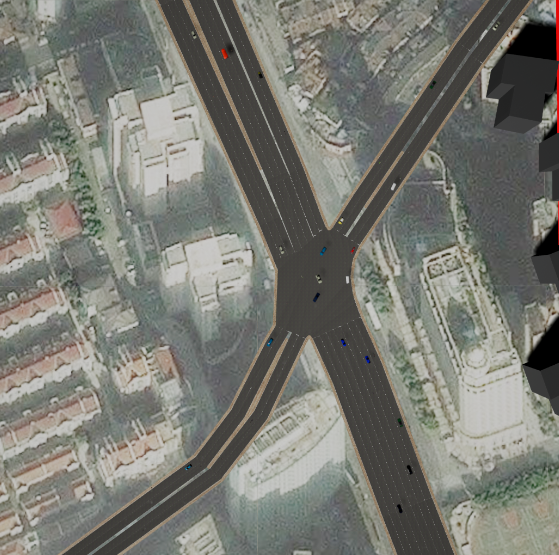}} 
        & \Fbox{\includegraphics[height=0.22\textwidth]{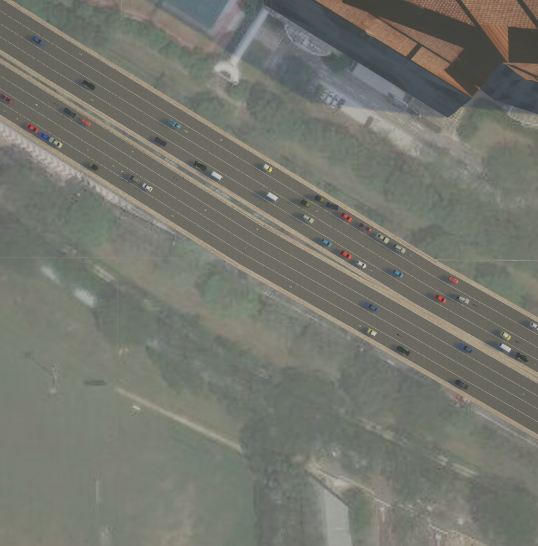}}  \\
        (a) & (b) 
    \end{tabular}
    
    \caption{Generalization over novel maps. Each map is populated with 110 traffic agents in our experiments. (a) Shanghai intersection. (b) Singapore highway.  
    }
    \label{fig:test_maps}
\end{figure}
\begin{figure*} [!t]
    \centering
    \begin{tabular}{ccc}
    \hspace*{-0.5cm}
    \includegraphics[width=0.34\textwidth]{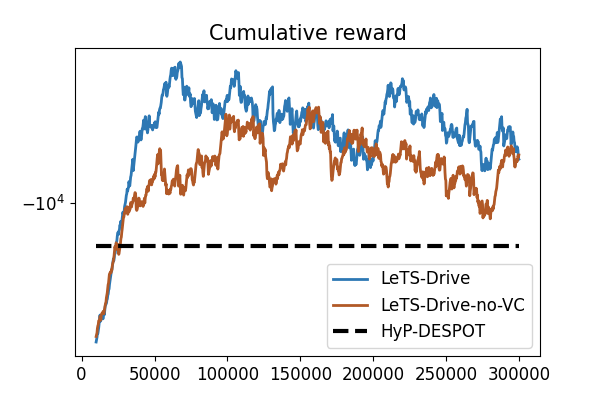} &
    \hspace*{-0.65cm}
    \includegraphics[width=0.34\textwidth]{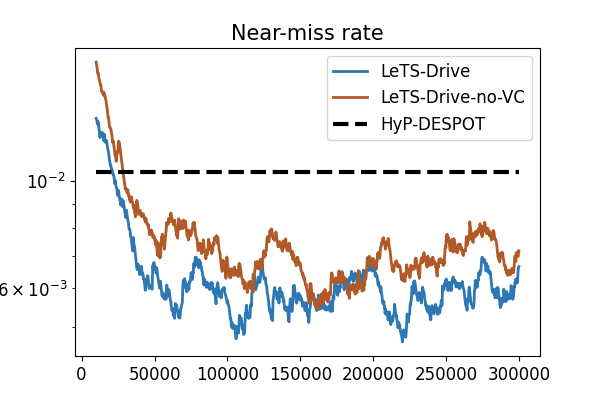} &
    \hspace*{-0.65cm}
    \includegraphics[width=0.34\textwidth]{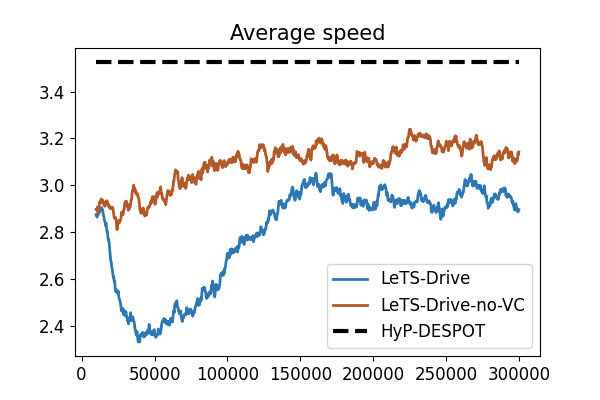} 
    \vspace*{-0.2cm}\\
    (a)&(b)&(c)
    \end{tabular}
    \caption{Performance of the \clreinforce planner policy with and without value clipping.}
    \label{fig:ablation_exp}
\end{figure*}

We now inspect the generalization of \algname. 
\subsubsection{Random test crowds}
Table \ref{tab:driving_results} shows the results for evaluating the trained planner and learner policies with unseen random crowds on the training map (Meskel intersection).
Numbers are calculated using more than $1000$ test episodes.

The results are generally consistent with those during training, clearly showing the benefits of integrating planning and learning from both directions. All \algname planner policies have drastically improved the rewards over HyP-DESPOT and their learning counterparts. 
Close-loop integration (\climitation and \clreinforce) has achieved significantly higher rewards than the open-loop (\olimitation), generating planner policies with the lowest near-miss rate and the highest rewards.

\rvs{\climitation and \clreinforce allow different trade-offs between exploitation and exploration, brought by the quality and the entropy of the learner policy, respectively. \climitation learns the best with a low policy entropy, because the entropy regularization objective often contradicts the imitation objective. It thus learns a strong but less explorative policy, which confidently guides the tree search towards high-quality directions. \clreinforce learns more stably with high policy entropy, due to the characteristics of policy gradient. It produces a weaker but more explorative policy, with approximately $18\%$ higher entropy, which encourages the search to explore a wider set of sensible actions.
As a result, we observe in Table \ref{tab:driving_results} that \climitation benefits equally from explicit planning and a strong learner policy, as shown by similar reward improvements w.r.t. the learner and POMDP planning. 
In contrast, \clreinforce improves over a much weaker learner policy, achieving similar integrated performance as \climitation, with the help of exploration.}

\subsubsection{Novel test maps}

 We further test \algname on two significantly different maps: another intersection in Shanghai and a highway in Singapore (\figref{fig:test_maps}). Results are shown in Table \ref{tab:driving_results_test_map}. Despite the extreme setup---training in a single intersection and testing on different maps---all variants of \algname have successfully generalized to the new environments, \rvs{almost matching the performance of \clreinforce trained specifically on the test maps, referred to as \clreinforce (Retrain) in Table \ref{tab:driving_results_test_map}}. Among them, \climitation and \clreinforce have achieved the best generalized performance. Closed-loop planning and learning brings the same level of benefits as in the training map,  delivering the safest planner policies with the highest rewards. 

\subsection{The Effect of Value Clipping}

Value clipping (\secref{sec:use_priors}) is an important algorithmic component that ensures the convergence of the guided belief tree search. We now show its practical effects.
\figref{fig:ablation_exp} shows the learning curves of \algname with and without value clipping. Without value clipping, the planner becomes overly optimistic due to the misuse of approximate heuristics. It seldom attempts to reduce the driving speed during training, thus inducing consistently higher near-miss rates. This compromises the reward throughout training.
In contrast, with value clipping, \algname becomes more cautious in driving, maintaining significantly lower speeds during the initial course of training. Afterward, the planner stably improves driving efficiency and constantly achieves higher rewards. This shows, besides maintaining theoretical guarantees, value clipping also enables more stable and efficient training in practice.

\section{Conclusion and Future Work}\label{sec:conclusion}
\rvs{We have presented the \algname algorithm, which integrates planning and learning by planning locally and learning globally in a closed loop.} \algname flexibly takes advantage of either self-supervised learning or reinforcement learning to learn heuristics for online planning. 
\rvs{Doing so, \algname  scales up online decision making under uncertainty}: it outperforms planning or learning alone, or open-loop integration of planning and learning.
Simulation experiments also show that \algname exhibits sophisticated driving behaviors
in challenging urban traffic with large heterogeneous crowds.




One limitation of LeTS-Drive is potential model errors. \clreinforce partially addresses the problem. It eliminates the bias in policy learning, but not in value learning, which still relies on self-supervision. Significant model errors may lead to inaccuracy in learned values and compromise the planner's  performance. Model learning, (\eg,  \cite{ayoub2020model}) alleviates this issue. It is also possible to apply reinforcement learning, \eg, temporal difference (TD) learning \cite{RLBookTD}, to learn the values directly, but it is undesirable,  because of sample inefficiency. 
Instead, we can refine value estimates through TD learning after ``warming up'' through self-supervision.

The current crowd-driving model in \algname can be further improved, by incorporating comprehensive traffic rules, social norms, \etc. 
With increased model complexity, we expect \algname to provide even more significant performance benefits through integrated planning and learning. 
There are also other models for different driving settings that are  interesting to consider \cite{ulbrich2013probabilistic,wray2017online,sunberg2020}.
More importantly, there is often a gap between the simulation and the real world for robot systems. Further research is required to study the effect of this gap on crucial issues, such as driving safety, as well as ways to close this gap~\cite{sim2real1,sim2real2,sim2real3}. \rvs{One may also perform human experiments in simulation to test the real-life performance of \algname, by letting human control some or all simulated exo-agents.}

Finally, \algname's core algorithmic ideas are not specific to \task, but  applicable in general to many large-scale, long-term planning tasks, such as object manipulation in clutter, multi-agent coordination, \etc. 
We will explore these exciting directions as our next step. 

\newpage
\clearpage
\bibliographystyle{ieeetr}
\bibliography{main}
\appendix
\subsection{Factored reward model} \label{sec:factored_reward}
The raw reward function described in \secref{sec:raw_reward} is sufficient for planning, but is particularly problematic for value learning due to the existence of rare but critical events, \eg, colliding with others. Particularly, this reward function is smooth at safe belief states, but can change dramatically at proximity to the critical events. 
To facilitate value learning, we factor our reward function, and consequently the value function, into safe-driving rewards $R_s$ and collision penalties $R_c$:
\begin{eqnarray}
    R &=& R_{s} + R_{c} \\
    R_{s} &=& R_{v} + R_{acc} + R_{change} \\
    R_{c} &=& R_{col}
\end{eqnarray} 
where the speed penalty $R_{v}$, smoothness penalties $R_{acc}$ and $R_{change}$, and collision penalty $R_{col}$ are defined as in \secref{sec:raw_reward}.  

To compute factored values from this reward function, we simply need to record the safe factor $\polvalue_s$ and collision factor $\polvalue_c$ separately during the backup process in the belief tree search. Particularly, at a belief node $b$, the Bellman's operator is executed as:
\begin{eqnarray}
\act^*=\argmax_{\act \in A}\left\{ \rfun{\node}{\act}+\gamma
	\sum_{\obs \in \obsset_{\node,\act}}p(\obs|\node,\act)\polvalue(\newnode)\right\} \label{eqn:backup1} \\
\polvalue_{s}(b) = R_{s}(\node, \act^*) + \gamma
	\sum_{\obs \in \obsset_{\node,\act^*}}p(\obs|\node,\act^*)\polvalue_{s}(\newnode) \label{eqn:backup_safe} \\
\polvalue_{c}(b) = R_{c}(\node, \act^*) + \gamma
	\sum_{\obs \in \obsset_{\node,\act^*}}p(\obs|\node,\act^*)\polvalue_{c}(\newnode) \label{eqn:backup_col}
\end{eqnarray}
\noindent
Eqn. (\ref{eqn:backup1}) denotes the regular value backup process where the best value is chosen according to the original value estimates $\polvalue$. Then the factored values associated with this best action $a^*$ is backed-up to the parent (Eqn. (\ref{eqn:backup_safe}-\ref{eqn:backup_col})).

Factored values at the root node are extracted as supervision labels for the learner. As the two factors are frequently zero, we further decompose the extracted value labels to binary masks and non-zero values before feeding to the learner:
\begin{equation}
V = \mathds{1}_{|V_{s} \neq 0} * V_{s}^{-} + \mathds{1}_{|V_{c} \neq 0} * V_{c}^{-}
\end{equation}
\noindent
where $V_{s}^{-}$ and $V_{s}^{-}$ are non-zero, negative values.

\subsection{Loss functions for learners} \label{sec:losses}

\subsubsection{Supervision loss} \label{sec:spv_loss}
In self-supervised learners, the policy network $\pi_{\theta}$ and the value network $v_{\theta'}$ are trained separately using supervised learning using action, mask, and value labels output by the planner. Given a dataset $D$ of size $N$, the loss functions, $l(\theta, D)$ and $l(\theta', D)$, measure the errors in action and value predictions, respectively:
\begin{eqnarray} \label{equation::supervision_loss}
l(\theta, D) &=&  - \frac{1}{N}\sum_{i}^N \log \pi_{\theta}(\act^i | x_b^i) - \alpha H(\pi_{\theta}(\cdot|x_b^i)) \label{eqn:action_loss}\\
l(\theta', D) &= & l_{mask}(\theta', D) + l_{value}(\theta', D) 
\end{eqnarray}
\noindent
where
\begin{eqnarray} \label{equation::supervision_loss_1}
l_{mask}(\theta', D) &= & \frac{1}{N}\sum_{i}^N(m_{s}(x_b^i|\theta')-\mathds{1}_{|V_{s}^i \neq 0})^2 \label{eqn:mask_loss}\\\nonumber &+&
(m_{c}(x_b^i|\theta')-\mathds{1}_{|V_{c}^i \neq 0})^2 \\
l_{value}(\theta', D) &=& \frac{1}{N}\sum_{i}^N(\mathds{1}_{|V_{s}^i \neq 0} * v_{s}(x_b^i|\theta')-V_{s}^i)^2 \nonumber \\
&+& (\mathds{1}_{|V_{c}^i \neq 0} * v_{c}(x_b^i|\theta')-V_{c}^i)^2 \label{eqn:nonzero_loss}
\end{eqnarray}
Here, $x_b^i$ is the history state in the $i$th data point; $\act^i$, $V_{s}^i$, and $V_{c}^i$ are the action and value labels obtained from the planner; $m_{s}(x_b^i|\theta')$ and  $m_{c}(x_b^i|\theta')$ are the mask predictions from the value network; and $v_{s}(x_b^i|\theta')$ and $v_{c}(x_b^i|\theta')$ are the value predictions from the value network.

Eqn. (\ref{eqn:action_loss}) represents the cross-entropy loss \cite{CrossEntropy} of the output policy w.r.t. to action labels (the first term) \rvs{augmented with entropy regularization for the policy itself (the second term). The regularization factor $\alpha$ is tuned online using gradient descent to help maintain a given target entropy of the output policy.} This dynamic update rule of $\alpha$ is borrowed from SAC \cite{SAC}. In our implementation, we set the target entropy to be $0.98\log|A|$ (targeting at scattered distributions) initially, and gradually anneal it to $0.65\log|A|$ (targeting at more concentrated distributions). 
Eqn. (\ref{eqn:mask_loss}) defines the prediction loss of the binary masks applied on value factors. Finally, Eqn. (\ref{eqn:nonzero_loss}) defines the regression loss for the non-zero values. 

\subsubsection{Reinforcement loss} \label{sec:rnf_loss}
In the reinforcement learner, we use SAC \cite{SAC}, an off-policy policy-gradient algorithm, to train the policy network. Specifically, we use its discrete-action version presented in \cite{SAC-Discrete}. The loss function of the policy learner is:
\begin{equation}
J(\theta)=E_{x_b \sim D}\left[\pi_{\theta}\left(x_b\right)^{T}\left[\alpha \log \left(\pi_{\theta}\left(x_b\right)\right)-Q_{\phi}\left(x_b\right)\right]\right].
\end{equation}
\noindent
Here, $x_b$ is a sampled history state from the replay buffer;
$\pi_{\theta}$ is the policy network;  
\rvs{$\alpha$ is a dynamically-tuned regularization scalar controlling the target entropy of $\pi_{\theta}\left(x_b\right)$;}
and $Q_{\phi}$ is a Q-network trained in a soft-Q learning manner, serving as a differentiable surrogate objective.
The Q-network shares the same architecture as the policy network (\figref{fig:architectures}), but without the softmax applied to the output. Details of the discrete-action SAC can be found in \cite{SAC-Discrete}.

Note that for policy-gradient, we can not directly apply the reward function described in \secref{sec:raw_reward} because of the scale and sparsity of collision penalties. Instead, we use the following smooth reward function in SAC:
\begin{equation}
    R = 0.05\frac{v}{v_{max}} -0.025\mathds{1}_{lane\neq 0} - \frac{1}{9t_{c}^2}
\end{equation}
\noindent
where the first term encourages efficient driving, the second penalizes excessive lane changes, and the third term penalizes proximity to collision events according to the time-to-collision, $t_{c}$, estimated using a constant-velocity prediction model.

\end{document}